\documentclass[10pt]{article}

\long\def\cv#1{}
\long\def\av#1{#1}

\av{
  \usepackage{geometry}
  \geometry{letterpaper, width=5.5in, height=8.5in, top=1in, headheight=14pt}
  \usepackage{amsmath,amssymb,amsthm}
  \usepackage{natbib}
  \usepackage{hyperref}
  \urlstyle{rm}
  \newtheorem{theorem}{Theorem}
  \newtheorem{lemma}[theorem]{Lemma}
  \newtheorem{proposition}[theorem]{Proposition}
  
  \theoremstyle{definition}
  \newtheorem{definition}[theorem]{Definition}
  \newcommand{\theoremref}[1]{Theorem~\ref{#1}}
  \newcommand{\sectionref}[1]{Section~\ref{#1}}
  \newcommand{\lemmaref}[1]{Lemma~\ref{#1}}
  \newcommand{\tableref}[1]{Table~\ref{#1}}
  \newcommand{\floatconts}[3]{#2\label{#1}\par\centering#3}
}

\usepackage{longtable}
\usepackage{booktabs}
\usepackage{xcolor}
\usepackage{tikz}
\usepackage{wrapfig}
\usepackage{worldflags}
\usetikzlibrary{positioning, arrows.meta, fit, backgrounds, calc}

\newcommand{\ZZ}{\mathbb{Z}}

\title{Agentic Neurosymbolic Collaboration for\\ Mathematical Discovery:\\ A Case Study in Combinatorial Design}
\author{Hai Xia\thanks{Algorithms and Complexity Group, TU Wien, Vienna, Austria} \and
  Carla P.\ Gomes\thanks{Department of Computer Science, Cornell University, USA} \and
  Bart Selman\footnotemark[2] \and
  Stefan Szeider\footnotemark[1]}
\date{}

\begin{document}

\maketitle
\av{\thispagestyle{empty}}

\begin{abstract}
 We study mathematical discovery through the lens of
  neurosymbolic reasoning, where an AI agent powered by a large
  language model~(LLM), coupled with symbolic computation tools, and human
  strategic direction, jointly produced a new result in combinatorial
  design theory. The main result of our human-AI collaboration is a tight lower bound on the
  imbalance of Latin squares for the difficult case
  $n \equiv 1 \pmod{3}$. Analyzing the logs of multiple sessions and multiple days, we can identify the contributions of different components in the framework. On one hand, the AI agent can uncover hidden structures and generate hypotheses. On the other hand, the symbolic components, including computer algebra, constraint solvers, and simulated annealing algorithms, can provide rigorous verification and exhaustive enumeration. Importantly, the human can make the important decision, shifting the dead-end research question to another promising one. Our analysis reveals that multi-model reviews among frontier LLMs proved reliable for criticism and error detection but unreliable for constructive claims. In the end, the resulting human-AI mathematical contribution, a tight lower bound of $4n(n{-}1)/9$, is achieved via a new class of \emph{near-perfect permutations}. The bound was also formally verified in Lean, showing that our neurosymbolic systems can indeed produce genuine discoveries in pure mathematics.
\end{abstract}

\section{Introduction}
\label{sec:intro}

There is increasing attention to mathematical discoveries with the
help of neurosymbolic frameworks~\citep{garcez23}, from program search
with large language models finding new extremal
constructions~\citep{romeraparedes24}, AI-guided conjectures in knot
theory and representation theory~\citep{davies21}, and automated
geometry proving~\citep{trinh24}. All these kinds of frameworks use
similar paradigms, using the neural side for generating conjectures or
candidates and using symbolic tools for verification and refinement.
The neurosymbolic general idea is quite promising for mathematical
discovery, as LLMs can have creative intuitions regarding proofs and
conjectures, and the symbolic tools can do effective and rigorous
verification and computation.  However, in each of these cases the objective is
fixed from the outset: proving a stated conjecture, optimizing a known quantity, or search within a predefined space. In contrast, our case
study involves open-ended exploration where the hypothesis itself was
not given in advance or the objective is shifted according to the exploration results emerged from the interaction with the AI agent.

In general, how each component contributes to the final mathematical
discovery is ambiguous, like the sequence of dead ends, pivots, and
insights leading from an open question to some novel results. Even
though the post-hoc description can be traced back, it is not clear
how distinct contributions are made by neural reasoning, symbolic
computation, and human judgments, separately or collaboratively.
Without the detailed process analysis, it is challenging to reproduce
discovery workflows, not to mention designing more efficient workflows
in a principled methodology.

In this paper, we present a detailed case study that addresses this gap.  From the interaction logs spanning several days and multiple
sessions, we reconstruct the discovery of a tight lower bound on the
imbalance of Latin squares, a fundamental problem in combinatorial
design theory with applications in experimental
design~\citep{jones15,vanes93}. Our result settles an open question
for the imbalance of the Latin squares case $n \equiv 1 \pmod{3}$.  The discovery process involved an
\emph{AI agent}, an interactive LLM-based assistant with access
to different symbolic tools (computer algebra systems and constraint solvers) and human strategic direction. Through the analysis of contributions from different components of
the framework, we found the agent's primary strength was
\emph{uncovering hidden structure}. For example, in our setting, the
agent discovered a parity constraint from the numerical data that had
been produced by the symbolic solvers in the exhaustive search. For
the verification of some conjectures from the LLM, the symbolic tools
can work on the rigorous search by scaling and verifying the instances
across different orders. Human steering
contributed to the critical research pivot: changing the question from
``find objects with zero imbalance'' to ``characterize the minimum
positive imbalance.''  We also find that multi-model review among
frontier LLMs was reliable for criticism.

More specifically, our \emph{main contributions} are as follows.

\textbf{(1) An agentic neurosymbolic collaboration framework:} We propose an agentic neurosymbolic collaborative framework for mathematical discovery in which an  LLM-based agent orchestrates symbolic computation tools, parallel
    multi-model review, and a two-tier persistent memory system that enables multi-session research continuity,
    while a human researcher provides strategic direction.
    \textbf{(2) Mathematical discovery in combinatorial design:} We demonstrate the effectiveness of our framework on an open problem in combinatorial design theory. In particular, our framework led to a new research direction and a concrete  mathematical discovery in combinatorial design comprising: (i) the introduction of the concept of \emph{near-perfect permutations} for Latin squares; (ii) a theorem stating  a tight lower bound of
    $4n(n{-}1)/9$ on the imbalance of Latin squares for
    $n \equiv 1 \pmod{3}$ for \emph{near-perfect
    permutations} that achieve this bound, and (iii) empirical verification of their existence
    up to $n = 52$.
  \textbf{(3) Asymmetry in multi-model review:} We show that multi-model review among frontier LLMs is
    reliable for criticism and error detection but unreliable for
    constructive mathematical claims, an asymmetry with implications
    for neurosymbolic system design.
    \textbf{(4) Process-level analysis of discovery:}    We provide a fine-grained process-level analysis of the
    discovery, reconstructed from interaction logs, that identifies
    the distinct contributions of neural pattern
    recognition, symbolic computation, and human strategic judgment
    across five phases, showing
    all three components are required.

\av{
\begin{enumerate}
\item \emph{An agentic neurosymbolic collaboration framework.} We propose a framework for mathematical discovery in which an LLM-based agent orchestrates symbolic computation tools, parallel multi-model review, and a two-tier persistent memory system that enables multi-session research continuity without updating model weights, while a human researcher provides strategic direction.

\item \emph{Mathematical discovery in combinatorial design.} We demonstrate the effectiveness of our framework on an open problem in combinatorial design theory. Our framework led to a new research direction and a concrete mathematical discovery comprising: (i)~the introduction of the concept of \emph{near-perfect permutations} for Latin squares; (ii)~a theorem stating a tight lower bound of $4n(n{-}1)/9$ on the imbalance of Latin squares for $n \equiv 1 \pmod{3}$, with \emph{near-perfect permutations} that achieve this bound; and (iii)~empirical verification of their existence up to $n = 52$.

\item \emph{Asymmetry in multi-model review.} We show that multi-model review among frontier LLMs is reliable for criticism and error detection but unreliable for constructive mathematical claims, an asymmetry with implications for neurosymbolic system design.

\item \emph{Process-level analysis of discovery.} We provide a fine-grained process-level analysis of the discovery, reconstructed from interaction logs, that identifies the distinct contributions of neural pattern recognition, symbolic computation, and human strategic judgment across five phases. A counterfactual analysis shows that the discovery required all three components.
\end{enumerate}
} 

\section{Problem Setting}
\label{sec:setting}

An $n \times n$ \emph{Latin square}~$L$ is an array filled with $n$
different symbols such that each symbol occurs exactly once in each
row and each column.  Latin squares are fundamental objects in
combinatorics, with applications in many fields of experimental
design, where spatial balance among treatments is
important~\citep{jones15,vanes93}.

We denote rows, columns, and symbols by
$\ZZ_n = \{0,\dots,n{-}1\}$.  For a row~$r$ and symbol~$s$, we write
$\mathrm{pos}(r,s)$ for the column containing~$s$ in row~$r$.  The
\emph{distance} between rows~$r_1$ and~$r_2$ is
$d(r_1,r_2) = \sum_{s \in \ZZ_n} \lvert \mathrm{pos}(r_1,s) -
\mathrm{pos}(r_2,s) \rvert$, and the \emph{imbalance}
of~$L$~\citep{diaz17} is
\begin{equation}\label{eq:imbalance}
  I(L) = \frac{1}{3} \sum_{0 \le r_1 < r_2 \le n-1}
  \lvert 3 \cdot d(r_1,r_2) - n(n{+}1) \rvert.
\end{equation}
A Latin square is \emph{spatially balanced} (an SBLS)~\citep{gomes2004streamlined} when $I(L) =
0$.  This requires $3 \mid n(n{+}1)$, i.e.,
$n \not\equiv 1 \pmod{3}$~\citep{zheng22}.  For
$n \equiv 1 \pmod{3}$, the ideal distance $n(n{+}1)/3$ is not an integer, so perfect balance is impossible.  The imbalance of Latin squares has been studied computationally
by~\citet{smith05,diaz17,peruvemba23,voboril25,zheng22}. However, the minimum achievable imbalance for $n \equiv 1 \pmod{3}$ had remained open.

Our human-AI contribution is a tight lower bound on this Latin square imbalance. The discovery process was based on an
\emph{AI agent}, an interactive assistant powered by
Anthropic's Claude Opus~4.5 large language model, which operated in a terminal environment with full access to a file system, shell commands, and external tools. The agent used the following symbolic tools: (i)~a computer algebra system (using SageMath through the Model Context Protocol, MCP server\footnote{\url{https://pypi.org/project/mcp-sage/}})
for algebraic analysis; (ii)~a custom high-performance solver written in Rust for exhaustive enumeration of combinatorial objects; and (iii)~simulated annealing search implemented in Python with JIT compilation.  The human researcher interacted with the agent through natural-language dialogue, steering the research direction.

\begin{figure}[t]
\centering
\resizebox{0.9\columnwidth}{!}{%
\begin{tikzpicture}[
    >=Stealth,
    actor/.style={draw, thick, rounded corners=5pt,
                  minimum height=1cm, align=center, font=\small},
    data/.style={fill=yellow!8, draw=black!20, rounded corners=1.5pt,
                 font=\scriptsize\itshape, inner sep=3pt, align=center},
    flow/.style={->, thick, #1},
    flow/.default={black!55},
  ]


  \node[actor, fill=blue!12, minimum width=4.2cm] (human) at (0, 5.5)
    {\textbf{Human Researcher}};

  \node[actor, fill=orange!15, minimum width=4.2cm,
        minimum height=1.1cm] (agent) at (0, 2.8)
    {\textbf{LLM Agent}\\[-2pt]{\scriptsize Claude Opus 4.5}};

  \node[actor, fill=orange!8, minimum width=2.8cm,
        font=\scriptsize] (delib) at (-7.5, 2.8)
    {\textbf{Multi-model Review}\\{\scriptsize Parallel frontier LLM}\\[-1pt]{\scriptsize consultation}};

  \node[actor, fill=teal!12, minimum width=2.6cm] (rust) at (-4.5, -0.5)
    {\textbf{Rust Solver}\\[-2pt]{\scriptsize Exhaustive enumeration}};
  \node[actor, fill=teal!12, minimum width=2.6cm] (sage) at (0, -0.5)
    {\textbf{SageMath}\\[-2pt]{\scriptsize Exact algebra}};
  \node[actor, fill=teal!12, minimum width=2.6cm] (sa) at (4.5, -0.5)
    {\textbf{SA Search}\\[-2pt]{\scriptsize Stochastic optimization}};

  \node[actor, fill=violet!12, minimum width=3.2cm,
        minimum height=1.4cm] (mem) at (8, 2.8)
    {\textbf{Persistent Memory}\\[-2pt]{\scriptsize Project state file,}\\[-1pt]{\scriptsize knowledge base, handover}};


  \draw[flow={blue!50!black}]
    ([xshift=-6mm]human.south) --
    node[data, left=2mm] {Research directions,\\quality judgments}
    ([xshift=-6mm]agent.north);

  \draw[flow={blue!50!black}]
    ([xshift=6mm]agent.north) --
    node[data, right=2mm] {Hypotheses,\\conjectures, proofs}
    ([xshift=6mm]human.south);

  \draw[flow={orange!50!black}]
    ([yshift=3mm]agent.west) --
    node[data, above=2mm, pos=0.82] {Proof drafts, claims}
    ([yshift=3mm]delib.east);

  \draw[flow={orange!50!black}]
    ([yshift=-3mm]delib.east) --
    node[data, below=2mm, pos=0.18] {Criticism, error reports}
    ([yshift=-3mm]agent.west);

  \draw[flow={teal!60!black}]
    ([xshift=-10mm]agent.south) --
    node[data, left=2mm, pos=0.5] {Generated code,\\scripts, parameters}
    (rust.north);
  \draw[flow={teal!60!black}]
    (agent.south) -- (sage.north);
  \draw[flow={teal!60!black}]
    ([xshift=10mm]agent.south) -- (sa.north);

  \draw[flow={teal!60!black}, densely dashed]
    ([xshift=5mm]rust.north) -- ([xshift=-5mm]agent.south);
  \draw[flow={teal!60!black}, densely dashed]
    ([xshift=3mm]sage.north) -- ([xshift=3mm]agent.south);
  \draw[flow={teal!60!black}, densely dashed]
    ([xshift=-5mm]sa.north) --
    node[data, right=2mm, pos=0.5] {Results, data,\\counterexamples}
    ([xshift=7mm]agent.south);

  \draw[flow={violet!50!black}]
    ([yshift=3mm]agent.east) --
    node[data, above=2mm, pos=0.75] {Handover updates, new topics}
    ([yshift=3mm]mem.west);

  \draw[flow={violet!50!black}]
    ([yshift=-3mm]mem.west) --
    node[data, below=2mm, pos=0.25] {Project state, dead ends, knowledge}
    ([yshift=-3mm]agent.east);

\end{tikzpicture}%
}%
\caption{Collaboration architecture and data flow. Rounded boxes represent \emph{actors} (human, agent, symbolic tools); italic labels on arrows show the \emph{data} exchanged.  Solid arrows indicate instructions or programs. Dashed arrows indicate results or feedback. The AI agent manages all interactions between different components. It dispatches proofs to multiple frontier LLMs for critical reviews (left). Meanwhile, it maintains a three-component persistent memory system (right), consisting of a project state file, a searchable knowledge base, and a session handover protocol.}
\label{fig:architecture}
\end{figure}

\section{The Discovery Process}
\label{sec:process}

We describe the discovery from interaction logs spanning five days
and approximately fifteen sessions.  The process consists of five
phases as summarized below. The general collaboration follows the
architecture shown in \figurename~\ref{fig:architecture}, where the human researcher sets strategic goals and decides when to change
direction, and the AI agent translates these goals into executable
code for running the symbolic tools, and synthesizing results into
conjectures and proofs. On one hand, the data flows downward as generated code
and parameters. On the other hand, they flow upward as results, counterexamples, and
verification results. For supporting the long-persistent runs among multiple sessions where the new sessions begin at the previous ends, a three-component persistent memory system
(Section~\ref{sec:memory}) is designed to carry project state, documented dead
ends, and a searchable knowledge base.

We have five phases illustrating different facets of the discovery loop:
\emph{Phase~1} (Section~\ref{sec:deadend}) does algebraic exploration,
ending with a documented dead end;
\emph{Phase~2} (Section~\ref{sec:pivot}) is a human-initiated research
pivot, shifting from the previous goal to a more promising research
direction;
\emph{Phase~3} (Section~\ref{sec:eureka}) indicates how the agent
discovers the parity constraint from numerical data and then drafts an
elementary proof;
\emph{Phase~4} (Section~\ref{sec:formalization}) uses multiple-model
consultation to catch errors in the proof during the review; and
\emph{Phase~5} (Section~\ref{sec:extension}) extends the computational
verification by scaling up the orders of the near-perfect permutations.

\subsection{Dead End: Algebraic Reverse-Engineering}
\label{sec:deadend}

Our original goal was to find algebraic constructions for
\emph{perfect permutations}, permutations~$\sigma$ of~$\ZZ_n$ whose
shift correlation $f_\sigma(\delta) = \sum_{j} \lvert
\sigma(j{+}\delta) - \sigma(j) \rvert$ equals $n(n{+}1)/3$ for every
shift~$\delta$. Such permutations can result in spatially balanced
Latin squares by circulant constructions~\citep{zheng22,lebras12}.

By calling the Rust solver, the agent enumerated all perfect
permutations for $n = 12$. Then the agent did the algebraic analysis
with SageMath. But all the analyses showed there is no algebraic
construction for multiple $n$ except $n = 3$ and $5$, in the range
from $3$ to $17$. For all $n \ge 6$, perfect permutations end up
structureless in the symmetric group~$S_n$. Accordingly, we concluded
the algebraic approach could not get perfect permutations for all
orders.

\subsection{The Research Pivot}
\label{sec:pivot}

Even though we have negative results for the initial goal, constructing the solutions, we have two observations inspiring a change of research direction. The first observation
is that the exhaustive solver hit a combinatorial wall at $n = 18$:
even with symmetry breaking and multi-core parallelism, no perfect
permutation could be found within a one-hour timeout. Meanwhile, no perfectly balanced solutions can
exist for $n \equiv 1 \pmod{3}$.  Second, the algebraic analysis of
\sectionref{sec:deadend} showed that even \emph{known} perfect
permutations resist closed-form description.

The human researcher made the critical decision about changing the
research question entirely. The initial goal was to search for objects
with zero imbalance. However, the new conceptualized question is about
the minimum positive imbalance for different orders~$n$ with
$n \equiv 1 \pmod{3}$. Retrospectively, this switch of the research
question was decisive in the discovery loop, reframing the search
problem as an optimization problem and opening a new thread of
questions. 

\subsection{Parity from Data Anomaly}
\label{sec:eureka}

With the new question in hand, the agent computed a naive lower bound
for $n = 13$.  Since $n(n{+}1)/3 = 60.67\ldots$ is not an integer,
each distance~$d(r_1,r_2)$ must deviate from it by at least~$1/3$
after the scaling in~\eqref{eq:imbalance}, giving a naive lower bound
of $I \ge n(n{-}1)/6 = 26$.

However, when the agent ran the simulated annealing algorithm, the
best found imbalance was~$69.3$. That is more than two times above the
naive bound. This huge gap in the imbalance triggered a necessary
explanation about the theoretical minimum.

To this end, the agent investigated by examining the shift correlations
$f_\sigma(\delta)$ for the best permutation found.  The
observation: \textbf{$f_\sigma(\delta)$ was even for every
  shift~$\boldsymbol{\delta}$}. A check confirmed that \emph{all}
permutations exhibit this parity constraint, which was not specific to the particular permutation. 
The proof also came together:
\begin{enumerate}\itemsep0pt\parskip0pt
  \item Since $\lvert a - b \rvert \equiv a + b \pmod{2}$ for all
    integers, the distance $d(r_1,r_2) \equiv \sum_s
    (\mathrm{pos}(r_1,s) + \mathrm{pos}(r_2,s)) \equiv n(n{-}1)
    \equiv 0 \pmod{2}$.
  \item If every distance is even and the ideal value
    $n(n{+}1)/3$ is non-integer, then the minimum deviation per
    pair is not~$1$ but~$2$, doubling the lower bound.
  \item Using a careful case analysis with the sum constraint
    $\sum d = n^2(n^2{-}1)/6$, this yields $I \ge 4n(n{-}1)/9$.
\end{enumerate}
The agent did the verification by comparing the minimum imbalance
over all $576$~Latin squares of order~$4$ with the new bound. The
results matched exactly, so the new bound is confirmed.

The total time from posing the question to a verified proof was less
than one session, approximately one hour.  This was the shortest and
most consequential session of the project.

By computing the shift correlations for many permutations, the agent
could recognize the pattern (\emph{all values even}) that a human is
unlikely to discover without special focus.  The proof itself is not complex (four lines of modular arithmetic); however, finding it required the structure-discovery step that came from inspecting numerical data and preliminary lemmas (like \lemmaref{lem:fixedsum}).

\subsection{Formalization and Multi-Model Review}
\label{sec:formalization}

To guarantee the correctness of the proof, we tried to get the formal proof. However, the formalization exposed two errors that the initial discovery session had missed.

\paragraph{The circulant trap.} The first proof draft was actually based on the properties specific
to circulant Latin squares (the shift correlation
function~$f_\sigma$). However, Opus~4.5 claimed it holds for all Latin
squares. This error was found by calling multiple-model reviews, where
four frontier LLMs were consulted in parallel, and one of them
correctly identified that the generality of the proof does not hold
for all Latin squares. Therefore, the proof was fixed by proving the
bound purely from the Parity Lemma and the sum constraint, instead of
the circulant structure property.

\paragraph{The $\sigma$/$\sigma^{-1}$ mismatch.}  The circulant
construction $L[i][j] = (i + \sigma(j)) \bmod n$ has row distances
determined by~$\sigma^{-1}$, not by~$\sigma$.  This subtle error was
caught during formalization and resolved by redefining the
construction with $\tau = \sigma^{-1}$.

\paragraph{Multi-model reliability.}  The same multi-model
review that caught the circulant trap also confidently claimed
that the modular inversion map $\sigma(j) = j^{-1} \bmod p$ would
achieve $O(n^{5/2})$ imbalance.  Our experiments showed this claim
was wrong: inversion scales as $\Theta(n^{3.6})$.  This pattern
recurred throughout the project: \emph{multi-model review proved
  reliable for criticism but unreliable for constructive claims}.

\subsection{Computational Extension}
\label{sec:extension}

To match the lower bound, the agent did a minimum relaxation of the
perfect permutations that is compatible with the parity and
divisibility constraints. The relaxation is introduced as \emph{near-perfect
  permutations}: permutations~$\sigma$ with $f_\sigma(\delta) \in \{a,
a{+}2\}$ for all~$\delta$, where $a = (n(n{+}1){-}2)/3$.

A simulated annealing search found near-perfect permutations for all $n \equiv 1 \pmod{3}$ up to
$n = 52$, each achieving exactly $I = 4n(n{-}1)/9$. In contrast, the prior computational work did not reach the optimum even at $n=10$ as the results show in the recent paper~\citep{peruvemba23}. Furthermore, the agent also investigated whether algebraic constructions could
produce near-perfect permutations for large~$n$.  Its analysis of character sums over finite fields could show a fundamental barrier, proving the perfect imbalance cannot be closed by algebraic methods at all.

\subsection{Persistent Learning Across Sessions}
\label{sec:memory}

We equipped the agent with a structured \emph{persistent memory}
organized in two tiers following a \emph{progressive disclosure}
design.  The first tier is a project instruction file that the agent
reads at the start of every session.  This file records accomplished
goals, open questions, key formulas, tool usage patterns, and,
critically, documented dead ends.  At the end of each session, a
handover protocol updates this file with new results, failed
approaches, and next steps.  The second tier consists of a collection
of topic files, each tagged with keywords and a short context
description.  At session start, the agent receives only a compact
index of these files; the full content of any topic file is retrieved
on demand through keyword search.  This two-tier design makes
the context window of the agent relatively small at the beginning.
Meanwhile, it provides access to the entire accumulated knowledge base
when a specific topic becomes relevant.

This memory mechanism played a concrete role in the discovery.  When
the agent began the pivotal session (\sectionref{sec:eureka}), it had
access to a detailed record of the algebraic dead end: which
constructions had been tried, why they failed, and the explicit
instruction ``do not retry.''  This prevented the agent from
revisiting unproductive directions and focused its effort on the new
research question.  Similarly, the memory of all 672 perfect
permutations for $n = 12$, the Rust solver's performance
characteristics, and the exact formulas used in prior sessions
provided the agent with accumulated context that no single session
could have built from scratch.

From a neurosymbolic perspective, this persistent memory is a
form of \emph{external symbolic state} that compensates for the
agent's lack of long-term memory.  Each session, the LLM begins with
a fresh context window; the memory system augments sessions by
externalizing knowledge into structured, human-readable files that
both the agent and the researcher can inspect and modify.  This
design gives rise to a form of incremental learning that is purely
symbolic (no model weights are updated) yet provides the continuity
essential for multi-session research projects.

\section{The Mathematical Result}
\label{sec:result}

We now state the mathematical result precisely.  For a self-contained
presentation, we include full proofs of the main theorem; the
construction that achieves the bound appears as a corollary.

\subsection{The Lower Bound}

The argument is only based on the two properties of row-pair
distances that hold for every $n \times n$ Latin square.

\begin{lemma}[Fixed sum]\label{lem:fixedsum}
  For every $n \times n$ Latin square~$L$,
  $\sum_{r_1 < r_2} d(r_1,r_2) = n^2(n^2-1)/6$.
\end{lemma}

\begin{proof}
 \sloppypar For each symbol~$s$, the positions $\mathrm{pos}(r,s)$ form a
  permutation of~$\ZZ_n$, so $\sum_{r_1 < r_2} \lvert
  \mathrm{pos}(r_1,s) - \mathrm{pos}(r_2,s) \rvert = n(n^2{-}1)/6$.
  Summing over all $n$ symbols gives the result.
\end{proof}

\begin{lemma}[Parity]\label{lem:parity}
  For every $n \times n$ Latin square~$L$ and every pair of
  rows~$r_1, r_2$, the distance $d(r_1,r_2)$ is even.
\end{lemma}

\begin{proof}
  Since $\lvert a - b \rvert \equiv a + b \pmod{2}$ for all
  integers~$a,b$, we have
  $d(r_1,r_2) \equiv \sum_{s=0}^{n-1}
  (\mathrm{pos}(r_1,s) + \mathrm{pos}(r_2,s))
  \pmod{2}$.
  For each row~$r$, the values $\mathrm{pos}(r,0), \dots,
  \mathrm{pos}(r,n{-}1)$ form a permutation of~$\ZZ_n$, so
  $\sum_{s} \mathrm{pos}(r,s) = n(n{-}1)/2$.  Therefore
  $d(r_1,r_2) \equiv 2 \cdot n(n{-}1)/2 = n(n{-}1) \equiv 0
  \pmod{2}$.
\end{proof}

\begin{theorem}\label{thm:lower}
  For $n \equiv 1 \pmod{3}$, every $n \times n$ Latin square~$L$
  satisfies $I(L) \ge 4n(n{-}1)/9$.
\end{theorem}

\begin{proof}
  Let $N = \binom{n}{2}$ and $a = (n(n{+}1){-}2)/3$.  Writing
  $n = 3k{+}1$, one verifies that $a = 3k(k{+}1)$ is a nonnegative
  even integer and $N/3 = k(3k{+}1)/2$ is an integer.

  By \lemmaref{lem:parity}, each distance is even, so we write
  $d(r_1,r_2) = a + 2x_{r_1,r_2}$ for integers~$x_{r_1,r_2}$.
  \lemmaref{lem:fixedsum} gives $\sum_{r_1 < r_2} d(r_1,r_2) =
  N \cdot n(n{+}1)/3 = N(a + 2/3)$, hence
  \begin{equation}\label{eq:sumx}
    \textstyle\sum_{r_1 < r_2} x_{r_1,r_2} = N/3.
  \end{equation}
  Since $3a = n(n{+}1) - 2$, the deviation for each pair is
  $3\,d - n(n{+}1) = -2 + 6x$.  We claim the pointwise inequality
  \begin{equation}\label{eq:pointwise}
    \lvert{-2 + 6x}\rvert \ge 2 + 2x
    \quad \text{for every integer } x.
  \end{equation}
  Indeed: for $x \ge 1$, $6x-2 \ge 2+2x$ since $4x \ge 4$; for
  $x = 0$, both sides equal~$2$; for $x \le -1$, the left side is
  $2 - 6x \ge 8$ while $2 + 2x \le 0$.

  Summing~\eqref{eq:pointwise} over all row pairs and
  applying~\eqref{eq:sumx} yields
  $\sum_{r_1 < r_2} \lvert 3d - n(n{+}1) \rvert
  \ge 2N + 2N/3 = 8N/3$.
  Therefore
  $I(L) \ge 8N/(3 \cdot 3) = 8N/9 = 4n(n{-}1)/9$.
\end{proof}

\subsection{The Matching Construction}

The \emph{shift correlation} of a permutation
$\sigma \colon \ZZ_n \to \ZZ_n$ at shift~$\delta$ is
$f_\sigma(\delta) = \sum_{j \in \ZZ_n} \lvert \sigma(j {+} \delta) -
\sigma(j) \rvert$.  For the circulant Latin square $L[i][j] = (i +
\sigma^{-1}(j)) \bmod n$, the distance between rows at
offset~$\delta$ equals $f_\sigma(\delta)$.

\begin{definition}\label{def:nearpp}
  A permutation $\sigma$ of~$\ZZ_n$ is a \emph{near-perfect
    permutation} (near-PP) if $n \equiv 1 \pmod{3}$ and
  $f_\sigma(\delta) \in \{a, \; a{+}2\}$ for every
  $\delta \in \{1,\dots,n{-}1\}$, where $a = (n(n{+}1) - 2)/3$.
\end{definition}

\begin{proposition}\label{prop:nearpp}
  If $\sigma$ is a near-PP of order~$n$, then the circulant Latin
  square $L[i][j] = (i + \sigma^{-1}(j)) \bmod n$ has
  $I(L) = 4n(n{-}1)/9$.
\end{proposition}

\cv{According to the calculation results, the equality holds
in~\eqref{eq:pointwise} for $x \in \{0,1\}$.
calculation.}
\av{\begin{proof}
  Since $f_\sigma(\delta) \in \{a,\,a{+}2\}$, every row-pair distance
  satisfies $d = a + 2x$ with $x \in \{0,1\}$.  For both values,
  equality holds in~\eqref{eq:pointwise}: $|{-}2+6\cdot0|=2=2+2\cdot0$
  and $|{-}2+6\cdot1|=4=2+2\cdot1$.  Hence the lower bound
  $I(L) \ge 4n(n{-}1)/9$ from Theorem~\ref{thm:lower} is attained with
  equality.
\end{proof}}

\begin{theorem}\label{thm:existence}
  Near-perfect permutations exist for every
  $n \equiv 1 \pmod{3}$ with $4 \le n \le 52$.
\end{theorem}

\tableref{tab:results} summarizes the computational results.  Each
entry was found by simulated annealing and independently verified.

\begin{table}[htbp]
\floatconts
  {tab:results}%
  {\caption{Near-PP existence for $n \equiv 1 \pmod{3}$,
    $4 \le n \le 52$.  Each achieves the optimal imbalance
    $I^* = 4n(n{-}1)/9$.  Times are in seconds obtained on the machine equipped with CPU i7-1260P and 32 GB RAM under the operating system Ubuntu 24.04.3 LTS.}}%
  {\av{\medskip}\resizebox{\columnwidth}{!}{%
   \begin{tabular}{@{}rrr@{\quad}rrr@{\quad}rrr@{\quad}rrr@{}}
    \toprule
    $n$ & $I^*$ & Time (s) & $n$ & $I^*$ & Time (s) & $n$ & $I^*$ & Time (s) & $n$ & $I^*$ & Time (s) \\
    \midrule
    4  & $16/3$  & ${<}1$ & 19 & $152$    & 1 & 34 & $1496/3$ & 5   & 49 & $3136/3$ & 54 \\
    7  & $56/3$  & ${<}1$ & 22 & $616/3$  & 2 & 37 & $592$    & 5   & 52 & $3536/3$ & 32 \\
    10 & $40$    & ${<}1$ & 25 & $800/3$  & 2 & 40 & $2080/3$ & 74  &    &          &    \\
    13 & $208/3$ & ${<}1$ & 28 & $336$    & 2 & 43 & $2408/3$ & 139 &    &          &    \\
    16 & $320/3$ & 1      & 31 & $1240/3$ & 4 & 46 & $920$    & 31  &    &          &    \\
    \bottomrule
  \end{tabular}}}
\end{table}

\subsection{Formal Verification in Lean~4}

We formally verified the lower bound (\theoremref{thm:lower}) in the
Lean~4 proof assistant~\citep{demoura21} with the Mathlib library,
using the LeanBack MCP server\footnote{\url{https://pypi.org/project/leanback/}}.
The formalization includes the definition for Latin squares as a
function $\mathrm{Fin}\,n \to \mathrm{Fin}\,n \to \mathrm{Fin}\,n$
with row and column bijectivity, the two lemmas (the Parity Lemma and
the Fixed Sum Lemma), and the main inequality property. It is around
340~lines of Lean formalization code, and here is the key theorem statement where we use \verb|imbalance3| $= 3 \cdot I(L)$ to avoid rational arithmetic in the formalization.

\begin{verbatim}
theorem LatinSquare.imbalance_lower_bound
    (L : LatinSquare n) (hn : n % 3 = 1) :
    3 * L.imbalance3 >= 4 * ↑n * (↑n - 1)
\end{verbatim}

\section{Analysis of the Discovery Process}
\label{sec:analysis}

\subsection{Role Taxonomy}
\label{sec:taxonomy}

\tableref{tab:roles} maps each phase of the discovery to its primary
function and the components. Briefly, the \textbf{agent} uncovered
hidden structure (the all-even parity pattern) from the numerical data
of permutations and then generated hypotheses. Meanwhile, the critical
reviews from multiple models can catch errors. The \textbf{symbolic}
components, including the Rust solver, SageMath, and simulated
annealing, provided exhaustive enumeration and algebraic analysis,
confirming conjectures from the LLM. The \textbf{human} made a key
decision for the research question, shifting the current goal to a
more promising direction.

\begin{table}[htbp]
\floatconts
  {tab:roles}%
  {\caption{Contributions by component across discovery
    phases.  \textbf{A}~=~agent, \textbf{S}~=~symbolic
    (solvers, CAS), \textbf{H}~=~human.}}%
  {\av{\medskip}\begin{tabular}{@{}llc@{}}
    \toprule
    Phase / Step & Function & Role \\
    \midrule
    Enumerate PPs ($n \le 17$)
      & Exhaustive search & S \\
    Algebraic analysis
      & Pattern search (negative) & S \\
    Confirm dead end
      & Rigorous verification & S \\[\smallskipamount]
    Change research question
      & Strategic reframing & H \\[\smallskipamount]
    Compute naive bound
      & Numerical reasoning & A \\
    Notice $f(\delta)$ always even
      & Uncovering hidden structure & A \\
    Formulate the parity hypothesis
      & Abductive reasoning & A \\
    Prove Parity Lemma
      & Proof construction & A{+}S \\
    Prove a lower bound
      & Proof construction & A{+}S \\[\smallskipamount]
    Catch a circulant trap
      & Critical review & A \\
    Fix $\sigma$/$\sigma^{-1}$ error
      & Formal debugging & A{+}S \\[\smallskipamount]
    Find near-PPs ($n \le 52$)
      & Stochastic search & S \\
    Weil bound analysis
      & Theoretical reasoning & A \\
    \bottomrule
  \end{tabular}}
\end{table}


\subsection{Counterfactual Analysis}
\label{sec:counterfactual}

Without the agent, the parity observation would have been unlikely discovered. As
the pattern emerges only from systematic numerical exploration across
many orders (dimensions).  In the meanwhile, symbolic tools are the key for verifying the observation and checking proofs contained errors that only
formal checking exposed.  Without human steering, the research would
have remained in the algebraic dead end because the agent did not propose the
pivot resulting in the final discoveries.

\section{Related Work}
\label{sec:related}

Recent work on AI for mathematical discovery includes AI-guided
conjectures~\citep{davies21}, LLM-based program
search~\citep{romeraparedes24}, automated geometry
proving~\citep{trinh24}, conjecture generation on fundamental
constants~\citep{raayoni21}, and LLM-assisted formal
proving~\citep{yang24}.  We  focus on the \emph{process}
rather than on the end result; the central role of human steering sets
our work apart.  \citet{garcez23}~articulate the neurosymbolic vision
of integrating neural and symbolic reasoning.  
Evaluations of LLM mathematical capabilities~\citep{frieder24} are
consistent with our findings.  LLMs excel at pattern recognition but
struggle with rigor and strategic planning.

\section{Conclusion}
\label{sec:conclusion}
We have presented an agentic neurosymbolic collaboration framework for a genuine
mathematical discovery, a tight lower bound on Latin square
imbalance together with a new class of near-perfect
permutations.  Our analysis reveals that the discovery required all three components:
neural recognition to notice the parity constraint, symbolic
computation to verify and extend the result, and human judgment to
redirect the research when the initial approach stalled.

The failure modes are equally instructive.  Multi-model review
among frontier LLMs proved reliable for \emph{criticism} (catching
errors in proofs) but unreliable for \emph{constructive claims}.
The most important contribution of the human was recognizing the
degree of promise of the research question itself and making the
decisive shift from the current research question to another promising
one.

We see several possibilities for improving neurosymbolic discovery
systems.  First, integrating systematic verification into the agent's
reasoning loop could catch errors like the circulant trap earlier.
Second, developing meta-cognitive capabilities, the ability to
recognize unproductive research directions, could reduce dependence
on human steering.  Third, there are close ties between criticism and
construction in multi-model review and this asymmetry deserves
further study.  Understanding why LLM-based agents are better critics
than constructors has implications for both system design and training
objectives.

\paragraph{Supplementary material.}
The interaction-log excerpts and the full session timeline are also
provided as standalone supplementary files at
\url{https://doi.org/10.5281/zenodo.20530914}.

\subsection*{Acknowledgements}

This project is partially supported by an AI2050 Senior Fellowship, a Schmidt 
Sciences program, the National Science Foundation, the National Institute of Food and Agriculture and the Air Force Office of Scientific Research.

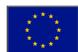
\begin{wrapfigure}[3]{l}{25pt}{\worldflag[width=18pt]{EU}}
\end{wrapfigure}
\noindent This project is also partially supported by the European Union's Horizon 2020 research and innovation programme
under the Marie Skłodowska-Curie grant agreement No.~101034440, and by the Austrian Science Fund (FWF) within the Cluster of Excellence Bilateral Artificial Intelligence (10.55776/COE12) and 10.55776/P36420. This research is supported by the Marshall Plan Foundation as part of the Marshall Plan Scholarship.

\av{\bibliographystyle{abbrvnat}}
\bibliography{refs-discovery}

\begin{thebibliography}{17}
\providecommand{\natexlab}[1]{#1}
\providecommand{\url}[1]{\texttt{#1}}
\expandafter\ifx\csname urlstyle\endcsname\relax
  \providecommand{\doi}[1]{doi: #1}\else
  \providecommand{\doi}{doi: \begingroup \urlstyle{rm}\Url}\fi

\bibitem[Davies et~al.(2021)Davies, Velickovic, Buesing, Blackwell, Zheng,
  Tomasev, Tanburn, Battaglia, Blundell, Juh{\'{a}}sz, Lackenby, Williamson,
  Hassabis, and Kohli]{davies21}
A.~Davies, P.~Velickovic, L.~Buesing, S.~Blackwell, D.~Zheng, N.~Tomasev,
  R.~Tanburn, P.~W. Battaglia, C.~Blundell, A.~Juh{\'{a}}sz, M.~Lackenby,
  G.~Williamson, D.~Hassabis, and P.~Kohli.
\newblock Advancing mathematics by guiding human intuition with {AI}.
\newblock \emph{Nat.}, 600\penalty0 (7887):\penalty0 70--74, 2021.
\newblock \doi{10.1038/S41586-021-04086-X}.
\newblock URL \url{https://doi.org/10.1038/s41586-021-04086-x}.

\bibitem[d'Avila Garcez and Lamb(2023)]{garcez23}
A.~d'Avila Garcez and L.~C. Lamb.
\newblock Neurosymbolic {AI:} the 3rd wave.
\newblock \emph{Artif. Intell. Rev.}, 56\penalty0 (11):\penalty0 12387--12406,
  2023.
\newblock \doi{10.1007/S10462-023-10448-W}.
\newblock URL \url{https://doi.org/10.1007/s10462-023-10448-w}.

\bibitem[de~Moura and Ullrich(2021)]{demoura21}
L.~de~Moura and S.~Ullrich.
\newblock The lean 4 theorem prover and programming language.
\newblock In A.~Platzer and G.~Sutcliffe, editors, \emph{Automated Deduction -
  {CADE} 28 - 28th International Conference on Automated Deduction, Virtual
  Event, July 12-15, 2021, Proceedings}, volume 12699 of \emph{Lecture Notes in
  Computer Science}, pages 625--635. Springer, 2021.
\newblock \doi{10.1007/978-3-030-79876-5\_37}.
\newblock URL \url{https://doi.org/10.1007/978-3-030-79876-5\_37}.

\bibitem[D{\'{\i}}az et~al.(2017)D{\'{\i}}az, Bras, and Gomes]{diaz17}
M.~D{\'{\i}}az, R.~L. Bras, and C.~P. Gomes.
\newblock In search of balance: The challenge of generating balanced latin
  rectangles.
\newblock In D.~Salvagnin and M.~Lombardi, editors, \emph{Integration of {AI}
  and {OR} Techniques in Constraint Programming - 14th International
  Conference, {CPAIOR} 2017, Padua, Italy, June 5-8, 2017, Proceedings}, volume
  10335 of \emph{Lecture Notes in Computer Science}, pages 68--76. Springer,
  2017.
\newblock \doi{10.1007/978-3-319-59776-8\_6}.
\newblock URL \url{https://doi.org/10.1007/978-3-319-59776-8\_6}.

\bibitem[Frieder et~al.(2023)Frieder, Pinchetti, Chevalier, Griffiths,
  Salvatori, Lukasiewicz, Petersen, and Berner]{frieder24}
S.~Frieder, L.~Pinchetti, A.~Chevalier, R.~Griffiths, T.~Salvatori,
  T.~Lukasiewicz, P.~Petersen, and J.~Berner.
\newblock Mathematical capabilities of chatgpt.
\newblock In A.~Oh, T.~Naumann, A.~Globerson, K.~Saenko, M.~Hardt, and
  S.~Levine, editors, \emph{Advances in Neural Information Processing Systems
  36: Annual Conference on Neural Information Processing Systems 2023, NeurIPS
  2023, New Orleans, LA, USA, December 10 - 16, 2023}, 2023.
\newblock URL
  \url{http://papers.nips.cc/paper\_files/paper/2023/hash/58168e8a92994655d6da3939e7cc0918-Abstract-Datasets\_and\_Benchmarks.html}.

\bibitem[Gomes and Sellmann(2004)]{gomes2004streamlined}
C.~P. Gomes and M.~Sellmann.
\newblock Streamlined constraint reasoning.
\newblock In M.~Wallace, editor, \emph{Principles and Practice of Constraint
  Programming - {CP} 2004, 10th International Conference, {CP} 2004, Toronto,
  Canada, September 27 - October 1, 2004, Proceedings}, volume 3258 of
  \emph{Lecture Notes in Computer Science}, pages 274--289. Springer, 2004.
\newblock \doi{10.1007/978-3-540-30201-8\_22}.
\newblock URL \url{https://doi.org/10.1007/978-3-540-30201-8\_22}.

\bibitem[Jones et~al.(2015)Jones, Woodward, and Stoller]{jones15}
M.~Jones, R.~Woodward, and J.~Stoller.
\newblock Increasing precision in agronomic field trials using {Latin} square
  designs.
\newblock \emph{Agronomy Journal}, 107\penalty0 (1):\penalty0 20--24, 2015.
\newblock \doi{10.2134/agronj14.0232}.

\bibitem[LeBras et~al.(2012)LeBras, Gomes, and Selman]{lebras12}
R.~LeBras, C.~P. Gomes, and B.~Selman.
\newblock From streamlined combinatorial search to efficient constructive
  procedures.
\newblock In J.~Hoffmann and B.~Selman, editors, \emph{Proceedings of the
  Twenty-Sixth {AAAI} Conference on Artificial Intelligence, July 22-26, 2012,
  Toronto, Ontario, Canada}, pages 499--506. {AAAI} Press, 2012.
\newblock \doi{10.1609/AAAI.V26I1.8147}.
\newblock URL \url{https://doi.org/10.1609/aaai.v26i1.8147}.

\bibitem[Raayoni et~al.(2021)Raayoni, Gottlieb, Manor, Pisha, Harris,
  Mendlovic, Haviv, Hadad, and Kaminer]{raayoni21}
G.~Raayoni, S.~Gottlieb, Y.~Manor, G.~Pisha, Y.~Harris, U.~Mendlovic, D.~Haviv,
  Y.~Hadad, and I.~Kaminer.
\newblock Generating conjectures on fundamental constants with the ramanujan
  machine.
\newblock \emph{Nat.}, 590\penalty0 (7844):\penalty0 67--73, 2021.
\newblock \doi{10.1038/S41586-021-03229-4}.
\newblock URL \url{https://doi.org/10.1038/s41586-021-03229-4}.

\bibitem[Ramaswamy and Szeider(2023)]{peruvemba23}
V.~P. Ramaswamy and S.~Szeider.
\newblock Proven optimally-balanced latin rectangles with {SAT} (short paper).
\newblock In R.~H.~C. Yap, editor, \emph{29th International Conference on
  Principles and Practice of Constraint Programming, {CP} 2023, Toronto,
  Canada, August 27-31, 2023}, volume 280 of \emph{LIPIcs}, pages 48:1--48:10.
  Schloss Dagstuhl - Leibniz-Zentrum f{\"{u}}r Informatik, 2023.
\newblock \doi{10.4230/LIPICS.CP.2023.48}.
\newblock URL \url{https://doi.org/10.4230/LIPIcs.CP.2023.48}.

\bibitem[Romera{-}Paredes et~al.(2024)Romera{-}Paredes, Barekatain, Novikov,
  Balog, Kumar, Dupont, Ruiz, Ellenberg, Wang, Fawzi, Kohli, and
  Fawzi]{romeraparedes24}
B.~Romera{-}Paredes, M.~Barekatain, A.~Novikov, M.~Balog, M.~P. Kumar,
  E.~Dupont, F.~J.~R. Ruiz, J.~S. Ellenberg, P.~Wang, O.~Fawzi, P.~Kohli, and
  A.~Fawzi.
\newblock Mathematical discoveries from program search with large language
  models.
\newblock \emph{Nat.}, 625\penalty0 (7995):\penalty0 468--475, 2024.
\newblock \doi{10.1038/S41586-023-06924-6}.
\newblock URL \url{https://doi.org/10.1038/s41586-023-06924-6}.

\bibitem[Smith et~al.(2005)Smith, Gomes, and Fern{\'{a}}ndez]{smith05}
C.~Smith, C.~P. Gomes, and C.~Fern{\'{a}}ndez.
\newblock Streamlining local search for spatially balanced latin squares.
\newblock In L.~P. Kaelbling and A.~Saffiotti, editors, \emph{IJCAI-05,
  Proceedings of the Nineteenth International Joint Conference on Artificial
  Intelligence, Edinburgh, Scotland, UK, July 30 - August 5, 2005}, pages
  1539--1540. Professional Book Center, 2005.
\newblock URL \url{http://ijcai.org/Proceedings/05/Papers/post-0460.pdf}.

\bibitem[Trinh et~al.(2024)Trinh, Wu, Le, He, and Luong]{trinh24}
T.~H. Trinh, Y.~Wu, Q.~V. Le, H.~He, and T.~Luong.
\newblock Solving olympiad geometry without human demonstrations.
\newblock \emph{Nat.}, 625\penalty0 (7995):\penalty0 476--482, 2024.
\newblock \doi{10.1038/S41586-023-06747-5}.
\newblock URL \url{https://doi.org/10.1038/s41586-023-06747-5}.

\bibitem[van Es and van Es(1993)]{vanes93}
H.~M. van Es and C.~L. van Es.
\newblock Spatial nature of randomization and its effect on the outcome of
  field experiments.
\newblock \emph{Agronomy Journal}, 85\penalty0 (2):\penalty0 420--428, 1993.
\newblock \doi{10.2134/agronj1993.00021962008500020046x}.

\bibitem[Voboril et~al.(2025)Voboril, Ramaswamy, and Szeider]{voboril25}
F.~Voboril, V.~P. Ramaswamy, and S.~Szeider.
\newblock Balancing latin rectangles with {LLM}-generated streamliners.
\newblock In M.~G. de~la Banda, editor, \emph{31st International Conference on
  Principles and Practice of Constraint Programming, {CP} 2025, Glasgow,
  Scotland, August 10-15, 2025}, volume 340 of \emph{LIPIcs}, pages
  36:1--36:17. Schloss Dagstuhl - Leibniz-Zentrum f{\"{u}}r Informatik, 2025.
\newblock \doi{10.4230/LIPICS.CP.2025.36}.
\newblock URL \url{https://doi.org/10.4230/LIPIcs.CP.2025.36}.

\bibitem[Yang et~al.(2023)Yang, Swope, Gu, Chalamala, Song, Yu, Godil, Prenger,
  and Anandkumar]{yang24}
K.~Yang, A.~M. Swope, A.~Gu, R.~Chalamala, P.~Song, S.~Yu, S.~Godil, R.~J.
  Prenger, and A.~Anandkumar.
\newblock Leandojo: Theorem proving with retrieval-augmented language models.
\newblock In A.~Oh, T.~Naumann, A.~Globerson, K.~Saenko, M.~Hardt, and
  S.~Levine, editors, \emph{Advances in Neural Information Processing Systems
  36: Annual Conference on Neural Information Processing Systems 2023, NeurIPS
  2023, New Orleans, LA, USA, December 10 - 16, 2023}, 2023.
\newblock URL
  \url{http://papers.nips.cc/paper\_files/paper/2023/hash/4441469427094f8873d0fecb0c4e1cee-Abstract-Datasets\_and\_Benchmarks.html}.

\bibitem[Zheng and Cao(2022)]{zheng22}
H.~Zheng and H.~T. Cao.
\newblock Perfect permutations and spatially balanced {Latin} squares.
\newblock \emph{Acta Mathematica Sinica, Chinese Series}, 65\penalty0
  (5):\penalty0 939--950, 2022.
\newblock \doi{10.12386/A20210055}.

\end{thebibliography}


\end{document}